\def\year{2022}\relax
\documentclass[letterpaper]{article} % DO NOT CHANGE THIS
\usepackage{aaai22}  % DO NOT CHANGE THIS
\usepackage{times}  % DO NOT CHANGE THIS
\usepackage{helvet}  % DO NOT CHANGE THIS
\usepackage{hyperref}       %

\usepackage{courier}  % DO NOT CHANGE THIS
\usepackage{natbib}  % DO NOT CHANGE THIS AND DO NOT ADD ANY OPTIONS TO IT
\usepackage{caption} % DO NOT CHANGE THIS AND DO NOT ADD ANY OPTIONS TO IT
\DeclareCaptionStyle{ruled}{labelfont=normalfont,labelsep=colon,strut=off} % DO NOT CHANGE THIS
\usepackage{xcolor}         % colors
\definecolor{light-gray}{gray}{0.5}

\usepackage[colorinlistoftodos]{todonotes}

\usepackage[symbol]{footmisc}

\usepackage{algorithm}
\usepackage{algorithmic}

\usepackage{newfloat}
\usepackage{listings}
\floatstyle{ruled}
\newfloat{listing}{tb}{lst}{}
\floatname{listing}{Listing}
\title{Conservative and Adaptive Penalty for Model-Based Safe Reinforcement Learning}

\author{%
Yecheng Jason Ma\equalcontrib$^{,1}$ \hspace*{20pt} Andrew Shen\equalcontrib$^{,2}$ \\ 
\textbf{Osbert Bastani$^1$ \hspace*{26pt} Dinesh Jayaraman$^1$}
}

\affiliations{
$^1$ University of Pennsylvania \hspace*{20pt} $^2$ University of Melbourne
}

\usepackage{amsmath}
\usepackage{amsthm}
\usepackage{amsfonts}
\usepackage{bm}
\usepackage{breqn}
\usepackage{breakcites}
\usepackage{amssymb}
\usepackage{algorithm}
\usepackage{algorithmic}
\usepackage{graphicx}
\usepackage{multirow}
\usepackage{listings}
\usepackage{wrapfig}
\usepackage{subfig}
\usepackage{mathtools}
\usepackage{bbm}
\usepackage[shortlabels]{enumitem}
\usepackage{caption}

\newcommand\norm[1]{\left\lVert#1\right\rVert}

\newcommand{\mc}{\mathcal}

\newcommand{\BR}{\mathbb{R}}

\newcommand{\BE}{\mathbb{E}}

\newtheorem{theorem}{Theorem}[section]
\newtheorem{corollary}[theorem]{Corollary}
\newtheorem{lemma}[theorem]{Lemma}
\theoremstyle{definition}

\newtheorem{assumption}[theorem]{Assumption}

\usepackage{ulem}

\definecolor{darkgreen}{rgb}{0,0.5,0}
\definecolor{darkblue}{rgb}{0,0,0.8}
\definecolor{darkred}{rgb}{0.9,0,0}

\newcommand{\drop}[1]{}

\newcommand{\edit}[1]{{\color{black} #1}}

\begin{document}

\maketitle

\begin{abstract}
Reinforcement Learning (RL) agents in the real world must satisfy safety constraints in addition to maximizing a reward objective. Model-based RL algorithms hold promise for reducing unsafe real-world actions: they may synthesize policies that obey all constraints using simulated samples from a learned model. However, imperfect models can result in real-world constraint violations even for actions that are predicted to satisfy all constraints. We propose Conservative and Adaptive Penalty (CAP), a model-based safe RL framework that accounts for potential modeling errors by capturing model uncertainty and adaptively exploiting it to balance the reward and the cost objectives. First, CAP inflates predicted costs using an uncertainty-based penalty. Theoretically, we show that policies that satisfy this conservative cost constraint are guaranteed to also be feasible in the true environment. We further show that 
this guarantees the safety of all intermediate solutions during RL training.  Further, CAP adaptively tunes this penalty during training using true cost feedback from the environment.  We evaluate this conservative and adaptive penalty-based approach for model-based safe RL extensively on state and image-based environments. Our results demonstrate substantial gains in sample-efficiency while incurring fewer violations than prior safe RL algorithms. Code is available at: \href{https://github.com/Redrew/CAP}{https://github.com/Redrew/CAP}
\end{abstract}

\section{Introduction} 

Many applications of reinforcement learning (RL) require the agent to satisfy safety constraints in addition to the standard goal of maximizing the expected reward. For example, in robot locomotion, we may want to impose speed or torque constraints to prevent the robot from damaging itself. Since the set of states that violates the imposed constraints is often a priori unknown, a central goal of \textit{safe} reinforcement learning \cite{pecka2014safe, garcia2015comprehensive} is to learn a reward-maximizing policy that satisfies constraints, while incurring as few constraint violations as possible during the agent's training process.

To reduce the cumulative number of constraint violations during training, a promising approach is to incorporate safety considerations into sample-efficient RL algorithms, such as model-based reinforcement learning (MBRL) \cite{sutton1990integrated, sutton1991planning}. MBRL refers to RL algorithms that use learned transition models to directly synthesize policies using simulated samples, thereby reducing the number of real samples needed to train the policy. Given the true environment transition model, it would be trivial to synthesize safe policies without any violations, since we could simply simulate a sequence of actions to evaluate its safety. However, MBRL agents must learn this transition model from finite experience, which induces approximation errors. In this paper, we ask: \textit{can safety be guaranteed during model-based reinforcement learning, despite these model errors?} We prove that this is indeed possible, and design a practical algorithm that permits model-based safe RL even in high-dimensional problem settings. %\textit{Can 

Specifically, we propose a model-based safe RL framework involving a \textbf{c}onservative and \textbf{a}daptive cost \textbf{p}enalty (\textbf{CAP}). We build on a basic model-based safe RL framework,  which simply executes a model-free safe RL algorithm inside a learned transition model. %Building on this, 
We make two important conceptual contributions to improve this basic approach. First, we derive a conservative upper bound on the error in the policy cost computed according to the learned model. In particular, we show that this error is bounded above by a constant factor of an integral probability metric (IPM) \cite{muller1997integral} computed over the true and learned transition models. Based on this bound, we propose to inflate the cost function with an uncertainty-aware penalty function. We prove that all feasible policies with respect to this conservative cost function, including the \textit{optimal} feasible policy (with highest task reward), are guaranteed to be safe in the true environment. A direct consequence is that we can ensure that all intermediate policies are safe and incur zero safety violations during training.

Second, this penalty function, though theoretically optimal, is often too conservative or cannot be computed for high-dimensional tasks. Therefore, in practice, we propose a heuristic penalty term that includes a scale hyperparameter to modulate the degree of conservativeness: higher scales produce behavior that is more averse to risks arising from modeling errors. Thus, different scales may be appropriate for use with different environments and model fidelities. 

We observe that this crucial scale hyperparameter need not be manually set and frozen throughout training. Instead, we can exploit the fact that the policy receives feedback on its true cost value from the environment, to formulate the entire inflated cost function as a control plant. In this view, the scale hyparparameter is the control input. Then, we can readily apply existing update rules from the control literature to tune the scale. In particular, we use a proportional-integral (PI) controller \cite{aastrom2006pid}, a simpler variant of a PID controller, to adaptively update the scale using cost feedback from the environment. 

Our overall CAP framework incorporates a conservative penalty term into predicted costs in the basic model-based safe RL framework, and adapts its scale to ensure the penalty is neither too aggressive nor too modest. To evaluate CAP, we first illustrate its proposed benefits in simple tabular gridworld environments using a linear programming-based instantiation of CAP; there, we show that CAP indeed achieves zero training violations and exhibits effective 
adaptive behavior. 
For state and image-based control environments, we evaluate a second instantiation of CAP, using a cost constraint-aware variant \cite{wen2020constrained} of cross entropy method  (CEM) \cite{de2005tutorial} coupled with state-of-art dynamics models \cite{chua2018deep, hafner2019learning} to optimize action sequences. Through extensive experiments, we show that our practical algorithms substantially reduce the number of real environment samples and unsafe episodes required to learn feasible, high-reward policies compared to model-free baselines as well as ablations of CAP. In summary, our main contributions are:
\begin{itemize}[leftmargin=*]
    \item an uncertainty-aware cost penalty function that can guarantee the safety of all training policy iterates
    \item an automatic update rule for dynamically tuning the degree of conservativeness during training.
    \item a linear program formulation of CAP that achieves near-optimal policies in tabular gridworlds while incurring zero training violation 
    \item and finally, scalable implementations of CAP that learn safe, high-reward actions in continuous control environments with high-dimensional states, including images. 
\end{itemize}

\section{Related Work} 

\paragraph{Safe RL} Our work is broadly related to the safe reinforcement learning and control literature; we refer interested readers to \cite{garcia2015comprehensive, brunke2021safe} for surveys on this topic. A popular class of approaches incorporates Lagrangian constraint regularization into the policy updates in policy-gradient algorithms \cite{achiam2017constrained, ray2019benchmarking, tessler2018reward, dalal2018safe, cheng2019end, zhang2020first, chow2019lyapunov}. These methods build on model-free deep RL algorithms \cite{schulman2017proximal, schulman2017trust}, which are often sample-inefficient, and do not guarantee that intermediate policies during training are safe. These safe RL algorithms are therefore liable to perform large numbers of unsafe maneuvers during training.

\paragraph{Model-Based Safe RL} Model-based safe RL approaches, instead, learn to synthesize a policy through the use of a transition model learned through data. A distinguishing factor among model-based approaches is their assumption on what is known or safe in the environment. Most works assume partially known dynamics \cite{berkenkamp2017safe, koller2019learningbased} or safe regions \cite{bastani2021safe, li2020robust, bansal2017hamilton, akametalu2014reachability}, and come with safety guarantees that are tied to these assumptions. In comparison, our work targets the more general setting, obtaining safety guarantees in a data-driven manner. In tabular MDP settings, we prove a high probability guarantee on the safety of any feasible solution under the conservative objective; we subsequently extend this result to ensure the safety of all training episodes. On more complex domains, we provide approximate and practically effective implementations for high-dimensional inputs, such as images, on which previous methods cannot be applied.

Our core idea of using uncertainty estimates as penalty terms to avoid unsafe regions has been explored in several prior works \cite{kahn2017uncertaintyaware, berkenkamp2017safe, zhang2020cautious}. However, our work provides the first theoretical treatment of the uncertainty-based cost penalty that is independent of the type of the cost (e.g., binary cost) and the parametric choice of the transition model. Our theoretical analysis is similar to that of \citet{yu2020mopo}, though we extend their results, originally in the offline constraint-free setting, to the online constrained MDP setting, and introduce a new result guaranteeing safety for all training episodes. Furthermore, our framework permits the cost penalty weight to automatically adjust to transition model updates, using environment cost feedback during MBRL training.

\section{Preliminaries} 
\label{sec:preliminaries}
In safe reinforcement learning, one common problem formulation is to consider an infinite-horizon constrained Markov Decision Process (CMDP)~\cite{altman1999constrained} $\mc{M} = (\mc{S}, \mc{A}, T, r, c, \gamma, \mu_0)$. Here, $\mc{S}, \mc{A}$ are the state and action spaces,  $T(s'\mid s,a)$ is the transition distribution, $r(s,a)$ is the reward function, $c(s,a)$ is the cost function, $\gamma \in (0,1)$ is the discount factor, and $s_0 \sim \mu_0(s_0)$ is the initial state distribution; we assume that both $r(s,a)$ and $c(s,a)$ are bounded. A policy $\pi:\mc{S} \rightarrow \Delta(\mc{A})$ is a mapping from state to distribution over actions. Given a fixed policy $\pi$, its state-action occupancy distribution is defined to be $\rho^\pi_T(s,a) \coloneqq (1-\gamma) \sum_{t=0}^\infty \gamma^t \text{Pr}^\pi(s_t=s, a_t=a)$, where $\text{Pr}^\pi(s_t=s, a_t=a)$ is the probability of visiting $(s,a)$ at timestep $t$ when executing $\pi$ in $\mc{M}$ starting at $s_0 \sim \mu_0$. The objective in this safe RL formulation is to find the optimal feasible policy $\pi^*$ that solves the following constrained optimization problem: 

\begin{equation}
    \label{eq:rl-objective}
    \begin{split}
    \max_{\pi}& \quad J(\pi) \coloneqq \BE\Big[\sum_{t=0} \gamma^t r(s_t,a_t) \Big] \\
    \text{s.t.}& \quad J_c(\pi) \coloneqq \BE\Big[\sum_{t=0} \gamma^t c(s_t,a_t) \Big] \leq C
    \end{split}
\end{equation}
where the expectation is over $s_0 \sim \mu_0(\cdot), s_t \sim T(s_t\mid s_{t-1},a_{t-1}), a_t \sim \pi(\cdot \mid s_t)$, and $C$ is a cumulative constraint threshold that should not be exceeded. We say that a policy $\pi$ is \textit{feasible} if it does not violate the constraint, and the optimization problem is feasible if there exists at least one feasible solution (i.e., policy). 

Unlike unconstrained MDPs, constrained MDPs cannot be solved by dynamic programming; instead, a common approach is to consider the dual of Eq~\eqref{eq:rl-objective}~\cite{altman1999constrained}:

\begin{equation}
\label{eq:rl-objective-dual-representation}
\resizebox{\columnwidth}{!}{
$\begin{split}
\max_{\rho(s,a) \geq 0} & \quad \frac{1}{1-\gamma} \sum_{s,a} \rho(s,a) r(s,a) \\
\text{s.t.}& \quad \frac{1}{1-\gamma} \sum_{s,a} \rho(s,a)c(s,a) \leq C \\ 
& \quad \sum_a \rho(s,a) = (1-\gamma)\mu_0(s) + \gamma \sum_{s',a'} T(s\mid s',a') \rho(s',a'), \forall s
\end{split}$}
\end{equation}
The dual problem Eq~\eqref{eq:rl-objective-dual-representation} is a linear program over occupancy distributions, and can be solved using standard LP algorithms; the second constraint defines the space of valid occupancy distributions by ensuring a ``conservation of flow" property among the distributions. Given its solution $\rho^*$, the optimal policy can be defined as $\pi^* (a \mid s) = \arg \max \rho^*(s,a)$, or equivalently, $\pi^*(a \mid s) = \rho^*(s,a)/\sum_a \rho^*(s,a)$ (if the optimal policy is unique).

Typically, the transition function $T$ is not known to the agent; thus, the optimal policy $\pi^*$ cannot be directly computed through LP. In model-based reinforcement learning (MBRL), the lack of known $T$ is directly addressed by learning an estimated transition function $\hat{T}$ through data $\mc{D} \coloneqq \{ (s,a,r,c,s')\}$. Then, we can define a \textit{surrogate} objective to Eq~\eqref{eq:rl-objective-dual-representation} by simply replacing $T$ with $\hat{T}$ and solving Eq~\eqref{eq:rl-objective-dual-representation} as before. Likewise, we can replace $J(\pi)$ with $\hat{J}(\pi)$, and $J_c(\pi)$ with $\hat{J}_c(\pi)$, to obtained model-based objectives in Eq~\eqref{eq:rl-objective}. Putting all this together, we may define a basic model-based safe RL framework \cite{berkenkamp2017safe, brunke2021safe} that iterates among  three steps: (1) solving for $\hat{\pi}^*$ approximately, (2) collecting data $(s,a,r,c,s')$ from $\hat{\pi}^*$, and (3) updating $\hat{T}$ using all collected data so far. However, at any fixed training iteration, the modeling error may lead to sub-optimal, potentially infeasible $\hat{\pi}^*$. This motivates our approach, described in the following sections. 

\section{CAP: Conservative and Adaptive Penalty} 

Next, we introduce \textbf{c}onservative and \textbf{a}daptive cost-\textbf{p}enalty (CAP), our proposed uncertainty and feedback-aware model-based safe RL framework. \edit{First, we precisely characterize the downstream effect of the model prediction error on the cost estimate $\hat{J}_c(\pi)$ by providing an upper bound on the true cost $J_c^*(\pi)$,} which allows us to derive a penalty function based on the epistemic uncertainty of the model. To this end, we adapt the return simulation lemma results in \cite{luo2021algorithmic, yu2020mopo} to the cost setting and derive the following upper bound on the true policy cost $\frac{1}{1-\gamma}\sum_{s,a}\rho^\pi_T(s,a)c(s,a)$ with respect to the estimated policy cost $\frac{1}{1-\gamma}\sum_{s,a}\rho^\pi_{\hat{T}}(s,a)c(s,a)$.

\subsection{\edit{Cost Penalty}}

First, given a policy mapping $\pi$, we define $V_c^{\pi}: \mc{S} \rightarrow \BR$ such that $V_c^{\pi}(s) \coloneqq \BE_{\pi, T} [\sum_{t=0}^\infty \gamma^t c(s_t,a_t) \mid s_0 = s]$. We make the following assumption on the realizability of $V_c^{\pi}$.

\begin{assumption}
\label{assumption:cost-realizability} 
There exists a $\beta > 0$ and a function class $\mc{F}$ such that $V^\pi_c \in \beta \mc{F}$ for all $\pi$.
\end{assumption}

With this assumption, we show that the difference \edit{between the estimated and true costs} can be bounded by the integral probability metric (IPM) defined by $\mc{F}$ computed between the true and the learned transition models. 

\begin{lemma}[Cost Simulation Lemma and Upper Bound] 
Let the $\mc{F}$-induced IPM be defined as 
\begin{equation}
    \label{eq:ipm}
    \begin{split}
    & d_{\mc{F}}(\hat{T}(s,a), T(s,a))  \\ 
    \coloneqq & \sup_{f \in \mc{F}} \lvert \BE_{s' \sim \hat{T}(s,a)}[f(s')] - \BE_{s' \sim T(s,a)}[f(s')] \rvert
    \end{split}
\end{equation}

Then, the difference between the expected policy cost computed using $T$ and $\hat{T}$ is bounded above:
\begin{equation}
\label{eq:cost-simulation-lemma}
\resizebox{\columnwidth}{!}{
$\sum _{s,a}\big(\rho^\pi_T(s,a) - \rho^\pi_{\hat{T}}(s,a)\big)c(s,a) \leq \gamma \beta \sum_{s,a} \rho^\pi_{\hat{T}}(s,a) d_{\mc{F}}\big(\hat{T}(s,a), T(s,a)\big)$
}
\end{equation}
\end{lemma}

We provide a proof in Appendix \ref{appendix:proofs}. This upper bound illustrates the risk of applying MBRL without modification in safety-critical settings. Attaining $\frac{1}{1-\gamma}\sum_{s,a}\rho^\pi_{\hat{T}}(s,a)c(s,a) \leq C$ does not guarantee that $\pi$ will be feasible in the real MDP (i.e., $\frac{1}{1-\gamma}\sum_{s,a}\rho^\pi_T(s,a)c(s,a) \leq C$) because the vanilla model-based optimization does not account for the model error's impact on the policy cost estimation, $\beta d_{\mc{F}}\big(\hat{T}(s,a), T(s,a)\big)$.

To enable model-based safe RL that can transfer feasibility from the model to the real world, for a fixed learned transition model $\hat{T}$, we seek a cost penalty function $u_{\hat{T}}: \mc{S} \times \mc{A} \rightarrow \BR$ such that $d_{\mc{F}}(\hat{T}(s,a), T(s,a)) \leq u_{\hat{T}}(s,a), \forall s,a$. If such a function exists, then we can solve the following LP:
\vspace{-0.1cm}
\begin{equation}
\label{eq:rl-objective-dual-representation-conservative}
\resizebox{\columnwidth}{!}{
$
\begin{split}
\max _{\rho(s,a) \geq 0 } ~& \frac{1}{1-\gamma} \sum_{s,a} \rho(s,a) r(s,a) \\
\text{s.t.}~& \frac{1}{1-\gamma} \sum_{s,a} \rho(s,a)\big(c(s,a) + \gamma \beta u_{\hat{T}}(s,a) \big) \leq C \\
& \sum_a \rho(s,a) = (1-\gamma)\mu_0(s) + \gamma \sum_{s',a'} \hat{T}(s\mid s',a') \rho(s',a'), \forall s
\end{split}$}
\end{equation}
We can guarantee that the solution policy $\pi$ of Eq~\eqref{eq:rl-objective-dual-representation-conservative} is feasible for $T$---in particular, note that
\begin{align*}
&\frac{1}{1-\gamma}\sum_{s,a}\rho^\pi_T(s,a)c(s,a) \\
&\leq \frac{1}{1-\gamma} \sum_{s,a} \rho^\pi_{\hat{T}}(s,a)\big(c(s,a) + \gamma \beta u(s,a) \big) \leq C.
\end{align*}
However, this result is not useful if we cannot compute $d_{\mc{F}}(\hat{T}(s,a), T(s,a))$. A suitable function class for analysis is $\mc{F} = \{f: \norm{f}_\infty \leq 1 \}$, which typically can be satisfied with Assumption \ref{assumption:cost-realizability} since the per-step cost is bounded. Then, for the tabular-MDP setting (i.e., finite state and action space), we can in fact obtain a strong probabilistic guarantee on feasibility.

\begin{theorem}[Tabular Case High-Probability Feasibility Guarantee]
\label{theorem:tabular-guarantee}
Assume $\mc{F} = \{f: \norm{f}_\infty \leq 1 \}$ and that Assumption \ref{assumption:cost-realizability} holds. Define $u(s,a) \coloneqq \sqrt{\frac{|\mc{S}|}{8n(s,a)} \ln \frac{4|\mc{S}||\mc{A}|}{\delta}}$, where $n(s,a)$ is the count of $(s,a)$ in $\mc{D}$ and $\delta \in (0, 1]$. Then, with probability $1-\delta$, a policy that is feasible for Eq~\eqref{eq:rl-objective-dual-representation-conservative} is also feasible for Eq~\eqref{eq:rl-objective-dual-representation}.
\end{theorem}

Furthermore, we can extend this result to guarantee that all intermediate solutions during training are safe. 
\begin{corollary}[High-Probability Zero-Training-Violations Guarantee]
\label{corollary:tabular-guarantee-training}
Assume the same set of assumptions as Theorem \ref{theorem:tabular-guarantee} and that the training lasts for $K$ episodes. Then, for any $\delta \in (0, 1]$, define $u(s,a) \coloneqq \sqrt{\frac{|\mc{S}|}{8n(s,a)} \ln \frac{4K|\mc{S}||\mc{A}|}{\delta}}$. Then, with probability $1-\delta$, all intermediate solutions to Eq~\eqref{eq:rl-objective-dual-representation-conservative} are feasible for Eq~\eqref{eq:rl-objective-dual-representation}.
\end{corollary}

Proofs are given in Appendix \ref{appendix:proofs}. At a high level, Theorem~\ref{theorem:tabular-guarantee} follows from observing that $d_{\mc{F}}$ is the total variation distance for the chosen $\mc{F}$ and applying concentration bound on the estimation error of $\hat{T}$. Then, Corollary~\ref{corollary:tabular-guarantee-training} can be shown by a union-bound argument.

Together, these results suggest that a principled way of incorporating a conservative penalty function into the 3-step basic model-based safe RL framework described at the end of Sec. \ref{sec:preliminaries} is to replace the original constrained MDP objective (i.e., Eq~\eqref{eq:rl-objective-dual-representation}) with its conservative variant (i.e., Eq~\eqref{eq:rl-objective-dual-representation-conservative}).

 \begin{algorithm}[tb]
	\caption{Safe MBRL with Conservative and Adaptive Penalty (CAP)} 
	\begin{algorithmic}[1]
	    \STATE {\bfseries Inputs:} Transition model $\hat{T}_\theta$, experience buffer $\mc{D}$, cost limit $C$, initial $\kappa$ value, $\kappa$ learning rate $\alpha$
	    \STATE Initialize $\mc{D}$ with random policy
		\FOR {$\text{Episode}=1,2,\ldots$}
		       \STATE {\color{light-gray} \# Conservative penalty}
		    \STATE Train $\hat{T}_{\theta}$ using $\mc{D}$
		    \STATE Optimize $\pi$ using Eq~\eqref{eq:rl-objective-dual-representation-conservative} (LP) or Eq~\eqref{eq:rl-objective-conservative} (CCEM)
		    \STATE Collect trajectory $\tau \coloneqq \{(s_t,a_t,r_t,c_t,s_{t+1})\}$ and store to buffer $\mc{D} = \mc{D} \cup \{\tau\}$
		    \STATE {\color{light-gray} \# Adaptive penalty}
			\STATE Compute $J_c(\pi_t) = \sum_{t=0} \gamma^t c_t$
			\STATE Update $\kappa \leftarrow \kappa + \alpha(J_c(\pi_t)-C)$
		\ENDFOR
	\end{algorithmic} 
\label{algo:cap-full}
\end{algorithm}

\subsection{Adaptive Cost Penalty}

\edit{The upper bound derived in the previous section can be overly conservative in practice.} Thus, we derive an adaptive penalty function based on environment feedback to make it more practical. First, we observe that
the conservative penalty modification described above is not yet enough for a practical algorithm, because the proposed penalty function as in the theorem or the corollary is too conservative, to the extent that Eq~\eqref{eq:rl-objective-dual-representation-conservative} might admit no solutions. In practice, it is often estimated as $u(s,a) \coloneqq \kappa/\sqrt{n(s,a)}$, where $\kappa \in \BR$ is some scaling parameter. 

We observe that setting $\kappa$ to a fixed value throughout training can lead to poor performance. Different scales may be appropriate for use with different environments, tasks, and stages of training. If it is set too low, then the cost penalty may not be large enough to ensure safety. On the other hand, if it is set too large, then the model may be overly conservative, discouraging exploration and leading to training instability. 

To avoid these issues, we propose to adaptively update $\kappa$ during training. Observe that the \textit{effect} of a particular $\kappa$ value on a policy's true cost in the environment can be measured from executing this policy in the real environment. Thus, we can in fact view the co-evolution of the policy and the learned transition model as a control plant, for which the policy cost is the control output; then, $\kappa$ can be viewed as its control input. Now, to set $\kappa$, we employ a PI controller, a simple variant of the widely used PID controller~\cite{aastrom2006pid} from classical control, to incrementally update $\kappa$ based on the current gap between the policy's true cost and the cost threshold. More precisely, we propose the following PI control update rule:
\begin{equation}
\label{eq:adaptive-penalty}
\kappa_{t+1} = \kappa_t + \alpha (J_c(\pi_t) - C)
\end{equation}
where $\alpha$ is the learning rate. 

This update rule is intuitive. Consider the direction of the $\kappa$ update when $J_c(\pi_t) < C$. In this case, the update will be negative, which matches our intuition that the cost penalty can be applied less conservatively due to the positive margin to the cost limit $C$. The argument for the case $J_c(\pi_t) > C$ is analogous. In high-dimensional environments, as the full expected cost cannot be computed exactly, and we instead approximate it using a single episode (i.e., the current policy $\pi_t$ rollout in the environment). To ensure $\kappa$ is non-negative, we additionally perform a $\max(0, \cdot)$ operation after each PI update.

\begin{figure*}[t!]
\centering
\textbf{Gridworld}\par\medskip
\includegraphics[width=1\textwidth]{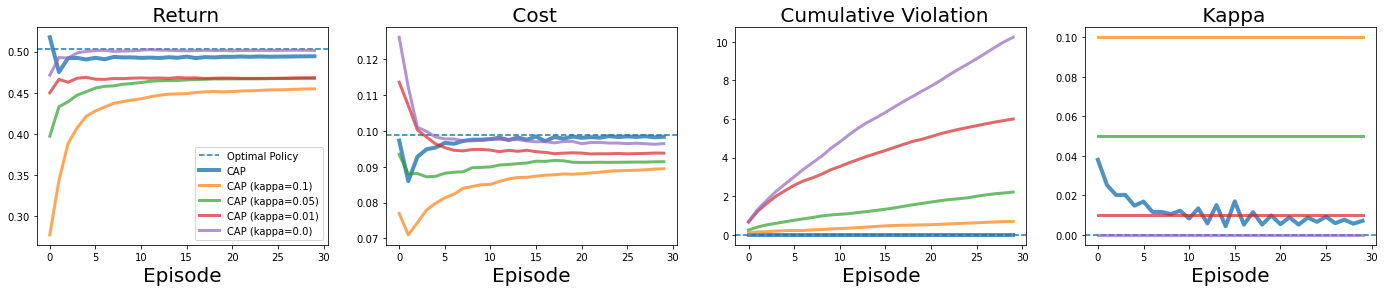}
\caption{\textbf{Tabular gridworld results.} CAP achieves near-optimal policy with zero constraint violations during training, while all ablations either converge to sub-optimal solutions or incur a high number of training violations.}
\label{figure:cap-gridworld}
\end{figure*}

Now, the full CAP approach is described in Algorithm \ref{algo:cap-full}. At a high level, CAP extends upon the basic model-based safe RL framework by (1) solving the conservative LP (Line 7, Eq~\eqref{eq:rl-objective-dual-representation-conservative}), and (2) adapting $\kappa$ using PI control (Lines 10 \& 11). We set the initial value for $\kappa$ using an exponential search mechanism, which we describe in the Appendix. We validate this LP formulation of CAP using a gridworld environment in our experiments. 

\subsection{CAP for High-Dimensional States}

Note that this tabular LP variant of CAP cannot extend to environments with continuous state and action spaces, representative of many high-dimensional RL problems of interest (e.g., robotics); their continuous nature precludes enumerating all state-action pairs, which is needed to express the linear program. Therefore, we propose a scalable implementation of CAP amenable to continuous control problems. First, we revert back to the policy-based formulation in Eq~\eqref{eq:rl-objective}, and define the following equivalent objective:
\begin{equation}
    \label{eq:rl-objective-conservative}
    \begin{split}
    \max_{\pi}& \quad \BE\Big[\sum_{t=0} \gamma^t r(s_t,a_t) \Big] \\
    \text{s.t.}& \quad \BE\Big[\sum_{t=0} \gamma^t \cdot \big(c(s_t,a_t) + \kappa u_{\hat{T}}(s_t,a_t) \big) \Big] \leq C
    \end{split}
\end{equation}
where $u(s_t,a_t)$ is a heuristic penalty function based on statistics of the learned transition model.

To optimize Eq~\eqref{eq:rl-objective-conservative}, we employ the constrained cross entropy method (CCEM) \cite{wen2020constrained, liu2021constrained} as our trajectory optimizer; the procedure is summarized in Algorithm \ref{algo:cap-ccem} in the Appendix. At a high level, CCEM first samples $N$ action sequences (Line 4) and computes their values and costs (Line 5). Then, if there were more than $E$ samples that satisfy the constraint, then the $E$ samples with highest rewards are selected (Line 10); otherwise, the $E$ samples with lowest costs are selected (Line 8). These selected \textit{elite} samples are used to update the sampling distribution (Line 12). This process continues for $I$ iterations, and the eventual distribution mean is selected as the optimal action sequence (Line 14). 

Next, we specify the choice of transition model and penalty function $u(s,a)$ for state-based and visual observation-based implementations, respectively. For the former, we model the environment transition function using an ensemble of size $N$, $\{\hat{T}^i_{\theta} = \mc{N}(\mu^i_{\theta}, \Sigma^i_{\theta}) \}_{i=1}^N$ \cite{chua2018deep} and set $u(s,a) = \max_{i=1}^N \norm{\Sigma^i_{\theta}(s,a)}_{\text{F}}$ to be the maximum Frobenius norm of the ensemble standard deviation outputs, as done for offline RL in \citet{yu2020mopo}. Our visual-based implementation builds on top of PlaNet \cite{hafner2019learning}, a state-of-art visual model-based transition model; here, we set $u(s,a)$ to be the ensemble disagreement of one-step hidden state prediction models \cite{sekar2020planning}. See the Appendix for details. In both tabular and deep RL settings, CAP adds negligible computational overhead, making it a practical safe RL algorithm.

\section{Experiments} 
CAP provides a general, principled, and practical framework for applying MBRL to safe RL. To support this claim, we comprehensively evaluate CAP against its ablations as well as model-free baselines in various environments. More specifically, we investigate the following questions:

\begin{enumerate}[label=(Q\arabic*), align=left, wide=0pt, leftmargin=*]
    \item Does CAP's theoretical guarantees approximately hold in tabular environments?
    \item Does CAP improve reward and safety upon its ablations (i.e., fixed $\kappa$ values)?
    \item Is CAP more sample and \textit{cost} efficient than state-of-art model-free baselines? 
    \item Can CAP learn safe policies even with high-dimensional inputs, such as images? 
\end{enumerate}

We investigate Q1-2 using a gridworld environment, and Q2-4 on two high-dimensional deep RL benchmarks. Our code is included in the supplementary materials. 

\subsection{Gridworld} 

We begin by validating our theoretical findings in tabular gridworld, where we can solve the constrained optimization problem (Eq~\eqref{eq:rl-objective-dual-representation-conservative}) exactly using standard LP algorithms. 

\subsubsection{Environment, Methods, Training Details} We consider an $8 \times 8$ gridworld with stochastic transitions; the reward and the cost functions are randomly generated Bernoulli distributions drawn according to a Beta prior. In addition to CAP, we compare against CAP ablations with fixed $\kappa$ values of $0, 0.01, 0.05$, and $0.1$; $\kappa=0$ corresponds to the basic MBRL approach without penalty. We also include the oracle LP solution computed using the true environment dynamics. For each method (except the oracle), the training procedure lasts $30$ iterations, in which each iterate includes (1) collecting $500$ samples using the current LP solution, (2) updating $\hat{T}$, and (3) solving the new conservative LP objective. See Appendix for more environment and training details. 

\subsubsection{Metrics \& Results} In Figure \ref{figure:cap-gridworld}, we illustrate the training curves for the return, cost, the cumulative number of intermediate policy solutions that violate the cost threshold. The first two metrics are standard, and the violation metric measures how safely a method explores. We additionally illustrate the training evolution of $\kappa$. These curves are the averages taken over the $100$ gridworld simulations; we defer standard deviation error bar to the Appendix for better visualizations except for the kappa curve.

As expected, CAP ($\kappa=0$), due to its asymptotically consistent nature, converges to the oracle as training continues; however, this comes at the cost of the highest number of training violations, precisely due to the lack of an uncertainty-aware penalty function. In sharp contrast, CAP is very close to the oracle in both reward and cost, and does so without incurring a single violation in all $100$ trials, as indicated by its flat horizontal line at $0$ in the violation plot. These results validate our key theoretical claims that when the cost penalty is applied properly, the resulting policy is guaranteed to be safe (Theorem \ref{theorem:tabular-guarantee}); furthermore, it applies to all intermediate policies during training (Corollary \ref{corollary:tabular-guarantee-training}),  answering Q1 above.

On the other hand, CAP ablations with fixed $\kappa$ values, though constraint-satisfying at the end, incur higher number of violations and achieve sub-optimal solutions, validated by their lower returns and conservative costs. Interestingly, while all these variants on \textit{average} satisfy the constraint from Episode 2 and on (i.e., their average costs are below the threshold of $0.1$ in the cost plot), their average numbers of violations uniformly increase throughout training. This suggests that fixed $\kappa$ values are not \textit{robust} to random gridworld simulations, as the same value may be too modest for some random draws and hence incur violations, and too aggressive for some other draws and lead to suboptimal solutions. Indeed, we observe greater variance in the performance of fixed $\kappa$ ablations than CAP (see Appendix). 

In contrast, CAP automatically finds suitable sequences of $\kappa$ for each simulation, evidenced by the large variance the $\kappa$ sequences exhibit over the simulations. Its zero-violation and lower variance in all metrics suggest that the adaptive penalty mechanism has the additional benefit of being \textit{distributionally robust} to the randomness in the environment distribution. Finally, the overall downward oscillating trend indeed reflects CAP's effectiveness at using feedback to optimize reward and cost simultaneously. Together, these ablations answer Q2 affirmatively. In the Appendix, we provide additional result analysis for this tabular experiment.

\subsection{High-Dimensional Environments}

Next, we evaluate CAP's generality and effectiveness in high-dimensional environments. We begin by summarizing our experimental setup; details are in the Appendix. 

\subsubsection{Environments}

\begin{figure}[t]
\begin{minipage}[t]{0.49\columnwidth}
\centering
\includegraphics[width=\textwidth]{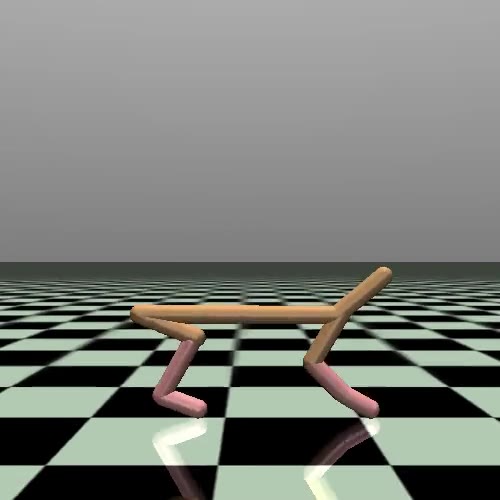}\par\medskip
\textbf{HalfCheetah}
\end{minipage}
\hfill
\begin{minipage}[t]{0.49\columnwidth}
\centering
\includegraphics[width=\textwidth]{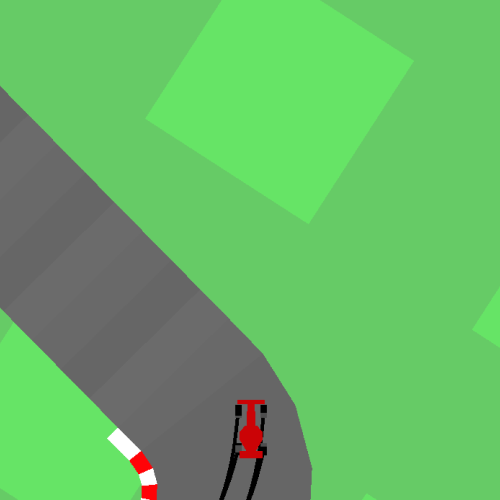}\par\medskip
\textbf{Car-Racing}
\end{minipage}
\caption{\textbf{High Dimensional Environments.}}
\label{figure:environments}
\end{figure}

We consider two deep RL environments, spanning different input modalities, constraint types, and associated cost types. We describe these environments here; see Figure \ref{figure:environments} for illustrations:
\begin{itemize}
\item A velocity-constrained version of Mujoco HalfCheetah \cite{todorov2012mujoco}, representative of robot tasks in which we want to avoid robots damaging themselves from over-exertion. The state space is 17-dimensional and the action space is 6-dimensional (controlling the robot joints). To ensure a meaningful cost constraint, we constrain the average velocity to be below $50\%$ of the average velocity of an unconstrained expert PPO agent ($152$) {\cite{zhang2020first}}.

\item A 2D image-based racing task Car-Racing \cite{brockman2016openai}, with randomized tracks in every episode. The state space is a $64 \times 64 \times 3$ top-down view of the car; the action space is continuous and 3-dimensional (steering, acceleration, and braking). A per-step cost of 1 is incurred if any wheels skid from excessive force; the cost limit is 0, indicating that the car should never skid. This task is representative of visual environments with complex dynamics.
\end{itemize}

\begin{table}
\resizebox{\columnwidth}{!}{
\begin{tabular}{|l|rrrr|}
\hline 
\multicolumn{1}{|c|}{\multirow{2}{*}{\textbf{Method}}}
& \multicolumn{4}{c|}{\textbf{HalfCheetah}} \\
& Steps & Return ($\uparrow$) & Cost (Limit 152) ($\downarrow$) & Violation ($\downarrow$)\\
\hline 
\textbf{Random}
& NA & -29.3 & 52.7 & NA \\
\hline
\textbf{PPO} 
& 1M   & 2791.3 & 296.9 & 378.0 \\ 
& 100K & 670.2  & 97.6  & 0 \\
\hline
\textbf{PPO-Lag}
& 1M   & 1436.8 & 150.7 & 108.0 \\
& 100K & 670.2  & 97.6  & 0 \\ 
\hline
\textbf{FOCOPS}
& 1M   & 1591.4 & 160.2 & 202.8 \\ 
& 100K & 456.0   & 84.6  & 0 \\
\hline
\textbf{CAP} (Ours)
& 100K & 1456.3 & 144.3 & 1.7 \\
\hline 
\multicolumn{1}{|c|}{\multirow{2}{*}{\textbf{Method}}}
& \multicolumn{4}{c|}{\textbf{Car-Racing}} \\ 
& Steps & Return ($\uparrow$) & Cost (Limit 0) ($\downarrow$) & Violation ($\downarrow$) \\
\hline 
\textbf{Random}
& NA & 3.9 & 159.3 & NA \\
\hline 
\textbf{PPO}
& 1M   & 32.7   & 52.0  & 975.0 \\ 
& 200K & 48.8   & 224.8 & 196.0 \\ 
\hline 
\textbf{PPO-Lag} 
& 1M   & -3.2   & 0.0   & 101.3 \\
& 200K & -3.2   & 0.3  & 101.3 \\ 
\hline 
\textbf{FOCOPS}
& 1M   & 23.4   & 0.8   & 581.0 \\ 
& 200K & 16.2   & 3.9   & 172.0 \\

\hline 
\textbf{CAP} & 200K & 21.7   & 0.4   & 93.3 \\
\hline 
\end{tabular}}
\caption{\textbf{Baseline comparison results.} CAP is substantially more sample-efficient with respect to both return and cost. In addition, it is much safer during training, as demonstrated by the significantly fewer violations.}
\label{table:model-free-comparison}
\end{table}

\begin{figure*}[t!]
\centering 
\includegraphics[width=0.9\textwidth]{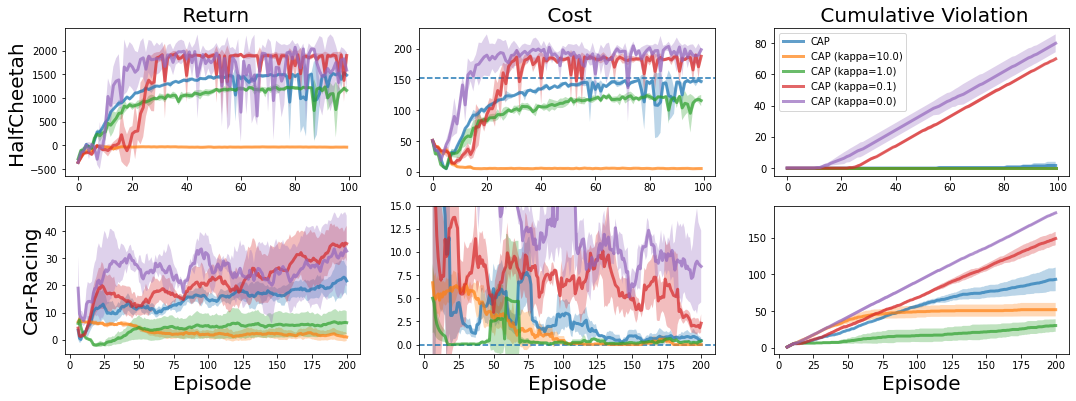}
\caption{\textbf{CAP Ablations on HalfCheetah (top) and Car-Racing (bottom).} The adaptive $\kappa$ achieves better balance than all fixed $\kappa$ values and incurs much fewer violations during training.}
\label{figure:cap-ablation}
\end{figure*}
\vspace{-0.2cm}
\subsubsection{Baselines}
In these high-dimensional settings, we compare against both deep model-free safe RL baselines as well as CAP ablations. To this end, we include PPO-Lagrangian (\textbf{PPO-LAG}), which iterates between PPO policy update and cost lagrangian parameter update to simultaneously optimize return and constraint satisfaction; despite its simplicity, PPO-LAG has been shown to be a strong safe RL baseline \cite{ray2019benchmarking}. Additionally, we include \textbf{FOCOPS} \cite{zhang2020cautious}, a state-of-art model-free algorithm which uses first-order projection methods to ensure that constraint satisfaction minimally deteriorates policy return. Finally, we include PPO \cite{schulman2017proximal} in order to provide comparison to an unconstrained method. Finally, as in the gridworld experiment, we consider CAP ablations with fixed $\kappa$ values and separately visualize the training curves. We use $\kappa=0, 0.1, 1, 10$ for both HalfCheetah and Car-Racing to include a wide range of magnitudes; in particular, $\kappa=0$ corresponds to the basic model-based safe RL approach without the conservative penalty; this is the constrained CEM method introduced by~\cite{liu2021constrained}. In Appendix \ref{appendix:high-dimensional-additional}, we additionally compare to PETS~\cite{chua2018deep}, a widely used unconstrained model-based planning method.

\subsubsection{Training Details \& Evaluation Metrics}
For model-free algorithms, we train using 1M environment steps, and for our model-based algorithms, we train using 100K steps for HalfCheetah and 200K steps for Car-Racing. An episode in both environments is 1000 steps. We report results on Car-Racing at 200k since it is more challenging to learn the dynamics of a visual environment; both model-free and model-based methods take more steps to converge in Car-Racing. As in gridworld, we report the training curves of the return, cost, and cumulative episodic violations; they are included in the Appendix. In the main text, we report a numerical ``snapshot" version of these curves at the end of training (average over last $10$ episodes); for model-free baselines, we also report these metrics after 100K/200K steps to have a head-to-head comparison against CAP. We include all hyperparameters and implementation details in the Appendix.

\subsubsection{Baseline Comparisons Results}
The results are shown in Table \ref{table:model-free-comparison}. While the most competitive algorithm FOCOPS matches CAP's return and cost with 1 million environmental steps in both environments, CAP requires about 5-10$\times$ fewer steps, demonstrating its sample efficiency. Furthermore, the relative performance of CAP at 100K/200K steps is significantly better than all model-free algorithms, which have not learned a good policy by that point. This has direct implication for safety. On the Car-Racing environment, because model-free methods learn much slower, they also spend more training episodes violating the constraint. On HalfCheetah, all methods achieve 0 cumulative episodic violations with 100K steps, but this is because in HalfCheetah the algorithm will not violate the speed constraint initially because it has not learned the running behavior yet. 

It is particularly illuminating to observe the cumulative episodic violations at the end of each method's training: we see that CAP violates the speed constraint in HalfCheetah for fewer than 2 episodes out of its 100 training episodes, while all baselines violate this constraint at much higher rates and volumes. This confirms that these model-free methods struggle to ensure safety during training regardless of the safety of their final policy, while CAP is able to minimize violations throughout training. On the more challenging image-based Car-Racing environment, CAP cannot avoid training violations entirely, but manages to significantly reduce them compared to the baselines. These comparisons provide strong evidence for Q3 and Q4.

\subsubsection{CAP Ablations Results} 
The training curves of CAP as well as its ablations are illustrated in Figure \ref{figure:cap-ablation}. Consistent with our findings in gridworld, setting $\kappa$ to a fixed value is rarely desirable. Setting it too low often leads to solutions that fail to satisfy constraint, suggested by the high training cost and violations of CAP ($\kappa=0.0, 0.1)$ in both environments; setting it too high often precludes reward learning in the first place, evidenced by the training curves of CAP ($\kappa=10.0)$ in both environments. Furthermore, since the cost limit is 0 on Car-Racing, exploration will always violate the constraint initially. Hence, we can additionally measure the safe exploration of a method by its slope on the violation curve: the lower the slope, the fewer violations a method incurs as training goes on. There, we see that CAP has the flattest violation slope out of all variants that learn policies with non-trivial driving behavior, answering Q2 affirmatively.

\vspace{-0.2cm}
\section{Conclusion}
We have presented CAP, a general model-based safe reinforcement learning framework. We have derived a linear programming formulation and proven that we can guarantee safety by using a conservative penalty; this penalty is then made adaptive based on environmental feedback to make it practically useful. We have validated our theoretical results in a tabular gridworld environment and demonstrated that CAP can be easily extended to high-dimensional visual environments through appropriate choices of optimizer and transition models. In future work, we aim to extend CAP to the offline and risk-sensitive settings \cite{yu2020mopo, ma2021conservative}. Overall, we believe that CAP opens many future directions in making MBRL practically useful for safe RL.
\clearpage 
\section*{Acknowledgement}
This work is funded in part by an Amazon Research Award, gift funding from NEC Laboratories
America, NSF Award CCF-1910769, NSF Award CCF-1917852 and ARO Award W911NF-20-1-0080. The U.S. Government
is authorized to reproduce and distribute reprints for Government purposes notwithstanding any
copyright notation herein
\bibliography{aaai22}

\clearpage 
\appendix 
\section{Proofs} 
\label{appendix:proofs}
In this section, we provide proofs for the theoretical results appeared in Section 4. We will restate each of the results and then append their corresponding proof. 

\begin{lemma}[Cost Simulation Lemma and Upper Bound] 
Let the $\mc{F}$-induced IPM be defined as

\begin{equation}
    \label{eq:ipm}
    \resizebox{\columnwidth}{!}{$
    d_{\mc{F}}(\hat{T}(s,a), T(s,a)) \coloneqq \sup_{f \in \mc{F}} \lvert \BE_{s' \sim \hat{T}(s,a)}[f(s')] - \BE_{s' \sim T(s,a)}[f(s')] \rvert$}
\end{equation}
Then, the difference between the expected policy cost computed using $T$ and $\hat{T}$ is bounded above:
\begin{equation}
\label{eq:cost-simulation-lemma}
\resizebox{\columnwidth}{!}{
$\sum _{s,a}\big(\rho^\pi_T(s,a) - \rho^\pi_{\hat{T}}(s,a)\big)c(s,a) \leq \gamma \beta \sum_{s,a} \rho^\pi_{\hat{T}}(s,a) d_{\mc{F}}\big(\hat{T}(s,a), T(s,a)\big)$}
\end{equation}
\end{lemma}

\begin{proof}
Using the telescoping lemma \cite{yu2020mopo, luo2021algorithmic}, we have that 
\begin{align*}
&\frac{1}{1-\gamma} \sum_{s,a}\big(\rho^\pi_T(s,a) - \rho^\pi_{\hat{T}}(s,a) \big)c(s,a) \\ 
=& \gamma \sum_{s,a} \rho^\pi_{\hat{T}}(s,a) \Big[ \BE_{s'\sim T(s,a)}V^\pi_{T}(s') - \BE_{s'\sim \hat{T}(s,a)} V^\pi_{\hat{T}}(s') \Big] 
\end{align*}
Then, by Assumption 4.1, we have that 
\begin{align*}
& \gamma \sum_{s,a} \rho^\pi_{\hat{T}}(s,a) \Big[ \BE_{s'\sim T(s,a)}V^\pi_{T}(s') - \BE_{s'\sim \hat{T}(s,a)} V^\pi_{\hat{T}}(s') \Big]  \\
\leq &\gamma \sum_{s,a} \rho^\pi_{\hat{T}}(s,a) \sup_{f \in \beta \mc{F}} \Big\lvert \BE_{s' \sim \hat{T}(s,a)}[f(s')] - \BE_{s' \sim T(s,a)}[f(s')] \Big\rvert \\ 
\leq & \gamma \sum_{s,a} \rho^\pi_{\hat{T}}(s,a) \beta d_{\mc{F}}(\hat{T}(s,a), T(s,a))
\end{align*}
Putting everything together, we have that 
\begin{align*}
    &\sum _{s,a}\big(\rho^\pi_T(s,a) - \rho^\pi_{\hat{T}}(s,a)\big)c(s,a)\\ 
    \leq &\gamma \beta \sum_{s,a} \rho^\pi_{\hat{T}}(s,a)  d_{\mc{F}}\big(\hat{T}(s,a), T(s,a)\big)
\end{align*}
\end{proof}
% TODO, the results in the main paper is missing the occupancy measure on the RHS, and also off by a factor of 1/(1-\gamma)... 

\begin{theorem}[Tabular Case High-Probability Feasibility Guarantee]
\label{theorem:tabular-guarantee-appendix}
Assume $\mc{F} = \{f: \norm{f}_\infty \leq 1 \}$ and that Assumption \ref{assumption:cost-realizability} holds. Define $u(s,a) \coloneqq \sqrt{\frac{|\mc{S}|}{8n(s,a)} \ln \frac{4|\mc{S}||\mc{A}|}{\delta}}$, where $n(s,a)$ is the count of $(s,a)$ in $\mc{D}$ and $\delta \in (0, 1]$. Then, with probability $1-\delta$, a policy that is feasible for Eq~\eqref{eq:rl-objective-dual-representation-conservative} is also feasible for Eq~\eqref{eq:rl-objective-dual-representation}.
\end{theorem}

\begin{proof}
In order for a policy that is feasible for Eq~\eqref{eq:rl-objective-dual-representation-conservative} is also feasible for Eq~\eqref{eq:rl-objective-dual-representation}, we need to have 
\begin{align*}
&\frac{1}{1-\gamma}\sum_{s,a}\rho^\pi_T(s,a)c(s,a) \\
&\leq \frac{1}{1-\gamma} \sum_{s,a} \rho^\pi_{\hat{T}}(s,a)\big(c(s,a) + \gamma \beta u(s,a) \big) \leq C.
\end{align*}
By the lemma above, this is equivalent to having $u(s,a) \geq d_{\mc{F}}(\hat{T}(s,a), T(s,a)), \forall s, a$. Since, we assume $\mc{F} = \{f: \norm{f}_\infty \leq 1 \}$, this implies
\begin{align*}
    &d_{\mc{F}}(\hat{T}(s,a), T(s,a)) \\
    =& d_{\text{TV}}(\hat{T}(s,a), T(s,a)) \\ 
    =& \frac{1}{2} \norm{\hat{T}(s,a), T(s,a)}_1
\end{align*}
where the last step follows because $\hat{T}(s,a)$ and $T(s,a)$ are multinomial distributions, which are countable. Then, we need
\begin{equation}
\label{eq:usa-bound}
u(s,a) \geq \frac{1}{2} \max_{s,a}\norm{\hat{T}(s,a), T(s,a)}_1
\end{equation}
By Hoeffding's inequality and the $l_1$ concentration bound for multinomial distribution, we have that, for any $\delta > 0$, we can set $u(s,a) \coloneqq \sqrt{\frac{|\mc{S}|}{8n(s,a)} \ln \frac{4|\mc{S}||\mc{A}|}{\delta}}$, then Eq~\eqref{eq:usa-bound} will hold with probability $1-\delta$, completing the proof.
\end{proof}

\begin{corollary}[High-Probability Zero-Training-Violations Guarantee]
\label{corollary:tabular-guarantee-training-appendix}
Assume the same set of assumptions as Theorem \ref{theorem:tabular-guarantee-appendix} and that the training lasts for $K$ episodes. Then, for any $\delta \in (0, 1]$, define $u(s,a) \coloneqq \sqrt{\frac{|\mc{S}|}{8n(s,a)} \ln \frac{4K|\mc{S}||\mc{A}|}{\delta}}$. Then, with probability $1-\delta$, all intermediate solutions to Eq~\eqref{eq:rl-objective-dual-representation-conservative} are feasible for Eq~\eqref{eq:rl-objective-dual-representation}.
\end{corollary}

\begin{proof}
Since we want all $K$ intermediate solutions to be feasible with probability $1-\delta$, the fault tolerance for any individual intermediate solution is $\delta / K$; this follows from an union bound argument. Therefore, we can adjust the concentration bound from Hoeffding's inequality by a factor of $K$ and obtain that by setting $u(s,a) \coloneqq \sqrt{\frac{|\mc{S}|}{8n(s,a)} \ln \frac{4K|\mc{S}||\mc{A}|}{\delta}}$, with probability $1-\delta$, we can guarantee all intermediate solutions to Eq~\eqref{eq:rl-objective-dual-representation-conservative} are feasible for Eq~\eqref{eq:rl-objective-dual-representation}.
\end{proof}

\section{CAP with Linear Programming}
This implementation of CAP is described in detail in the main text. Here, we describe the exponential search mechanism we use to initialize $\kappa$ for the very first training episode. Starting with a high value for $\kappa$ (e.g., 10), we use it to construct a new constrained optimization problem of form Eq~\eqref{eq:rl-objective-dual-representation-conservative} and attempt to solve it. If the problem is infeasible, then we halve the value of $\kappa$ and repeat the process. We stop at the first value of $\kappa$ for which the problem is feasible, and this value is taken as the initialized $\kappa$ value. 

\section{CAP with Constrained Cross Entropy Method}
In Algorithm \ref{algo:cap-ccem}, we provide the pseudocode for the constrained cross entropy method (CCEM). Here, we reiterate the algorithm description from the main text for completeness. At a high level, CCEM first samples $N$ action sequences (Line 4) and computes their values and costs (Line 5). Then, if there were more than $E$ samples that satisfy the constraint, then the $E$ samples with highest rewards are selected (Line 10); otherwise, the $E$ samples with lowest costs are selected (Line 8). These selected \textit{elite} samples are used to update the sampling distribution (Line 12). This process continues for $I$ iterations, and the eventual distribution mean is selected as the optimal action sequence (Line 14). 
\begin{algorithm}[htb]
	\caption{Constrained Cross Entropy Method} 
	\begin{algorithmic}[1]
	    \STATE {\bfseries Inputs:} Transition model estimate $\hat{T}_\theta$, experience buffer $\mc{D}$, cost limit $C$
	    \STATE {\bfseries CCEM Hyperparameters:} Population size $N$, elite population size $E$, max iteration $I$, planning horizon $H$, initial sampling distribution $\mc{N}(\mu_0, \Sigma_0)$
        \FOR {$i=1,\ldots,I$}
	        \STATE Sample $N$ action sequences $A^1 \coloneqq \{a_t^1\}_{t=1}^H, \ldots, A^N \coloneqq \{a_t^N\}_{t=1}^H \sim \mc{N}(\mu_{i-1}, \Sigma_{i-1})$
	        \STATE Evaluate the action sequences using Eq~\eqref{eq:rl-objective-conservative} by simulating trajectories in $\hat{T}_{\theta}$
            \STATE Construct feasible set $\mc{X} \coloneqq \{ A^n | \tilde{J}_c(A^n) \leq C, n \in [N] \}$
            \IF {$\lvert \mc{X} \rvert < E$}
                \STATE Construct elite set $\mc{E} \coloneqq \{$ The $E$ sequences out of all $\{A^n\}_{n=1}^N$ with lowest costs $\}$
            \ELSE
                \STATE Construct elite set $\mc{E} \coloneqq \{$The $E$ sequences in $\mc{X}$ with highest rewards $\}$
            \ENDIF
            \STATE Compute  $\mu_i, \Sigma_i$ using Maximum Likelihood over $\mc{E}$
		\ENDFOR
	    \STATE {\bfseries Outputs:} Optimal action sequence $\{a_1^*,...,a_H^* \} \coloneqq \mu_I$
	\end{algorithmic} 
\label{algo:cap-ccem}

\end{algorithm}

\section{Gridworld Experimental Detail} 
The gridworld environment is of size $8 \times 8$. The action space consists of the four directional primitives: Up, Down, Left, Right. For each action, there is a $20\%$ chance that slippage occurs and the agent moves in a random direction, introducing stochastic transitions to the environment. The reward and the cost functions are randomly generated Bernoulli distributions drawn according to a Beta(1,3) prior. Each state has uniform probability of being selected as the initial state for each episode. The discount rate is $0.99$. The cost threshold is kept at $0.1$ for all trials. Training lasts 30 episodes, and we use Gurobi \cite{gurobi} as the LP solver in our implementation. 

\begin{figure*}[t!]
\centering
\textbf{Gridworld}\par\medskip
\includegraphics[width=1\textwidth]{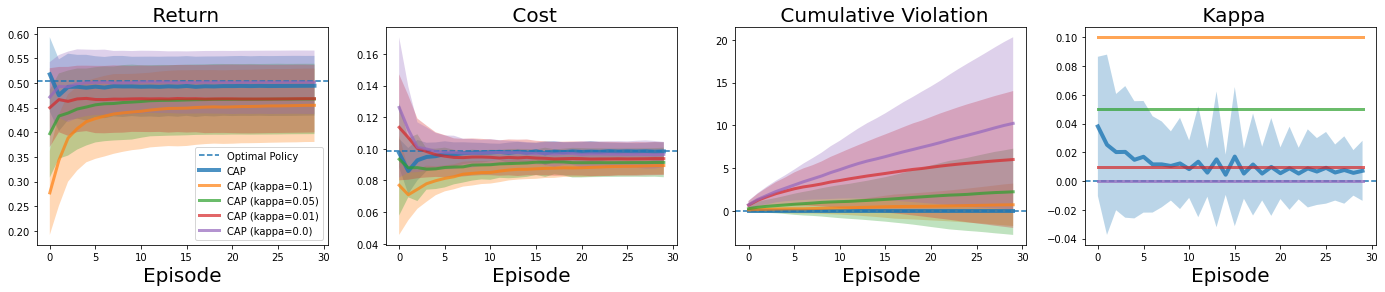}
\caption{\textbf{Tabular gridworld results with standard deviations.}}
\label{figure:cap-gridworld-std}
\end{figure*}

\begin{table*}
\centering
\begin{tabular}{l|rrrr|}
\hline 
\multicolumn{1}{c|}{\multirow{2}{*}{\textbf{Method}}}
& \multicolumn{4}{c|}{\textbf{Gridworld}} \\ 
% \hline
& Kappa $\kappa$& Return & Cost (Limit 0.1) & Violations   \\
\hline 
CAP
& Adaptive & 0.494 $\pm$ 0.061 & 0.098 $\pm$ 0.006 & 0.0 $\pm$ 0.0 \\
\hline 
CAP
& 0.1 & 0.454 $\pm$ 0.074 & 0.089 $\pm$ 0.0055 & 0.7 $\pm$ 2.48 \\ 
\hline 
CAP
& 0.05 & 0.468 $\pm$ 0.071 & 0.091 $\pm$ 0.0093 & 2.22 $\pm$ 5.05 \\ 
\hline 
CAP
& 0.01 & 0.468 $\pm$ 0.069 & 0.0938 $\pm$ 0.0104 & 6.0 $\pm$ 8.01 \\ 
\hline 
CAP
& 0.0 & 0.50 $\pm$ 0.064 & 0.0965 $\pm$ 0.0079 & 10.23 $\pm$ 10.08 \\
\hline 
\end{tabular}
\caption{\textbf{CAP ablations results on Gridworld.}} 
\label{table:cap-gridworld-ablation}
\end{table*}

\subsection{Additional results}
In Figure \ref{figure:cap-gridworld-std}, we illustrate the full version of Figure \ref{figure:cap-gridworld} with one standard deviation error bars added in. In 
Table \ref{table:cap-gridworld-ablation}, we also show these results in table format. As shown, CAP ablations with fixed $\kappa$ values exhibit greater variance in their performances over $100$ random seeds; this supports the claim in the main text that fixed $\kappa$ values are more sensitive to the randomness in the environment distribution. Finally, we observe that CAP on average obtains higher return than the optimal policy in the first iteration. At initialization, CAP is not guaranteed to satisfy the constraints, and it may optimize a constraint-violating policy
that achieves higher return than the optimal policy, explaining this behavior.

\section{High-Dimensional Environments Experimental Detail}

\begin{figure*}[t!]
\centering
\includegraphics[width=\textwidth]{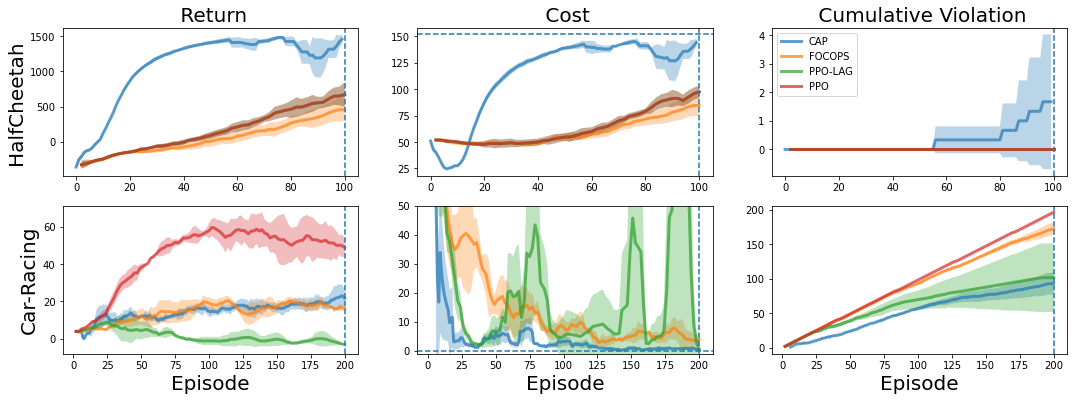}
\caption{\textbf{Step 100k/200k CAP and model-free baselines results on HalfCheetah (top) and Car-Racing (bottom).}}
\label{figure:cap-high-dimensional}
\end{figure*}

\begin{table*}
\centering
\begin{tabular}{l|rrrr|rrrr}
\hline 
\multicolumn{1}{c|}{\multirow{2}{*}{\textbf{Method}}}
& \multicolumn{4}{c|}{\textbf{HalfCheetah}}
& \multicolumn{4}{c}{\textbf{Car-Racing}} \\ 
& Kappa $\kappa$& Return & Cost (Limit 152) & Cost Violation & Kappa $\kappa$ & Return & Cost (Limit 0) & Cost Violation  \\
\hline 
CAP
& Adaptive & 1456.3 & 144.3 & 1.7
& Adaptive & 21.7 & 0.4 & 93.3 \\ 
\hline 
CAP
& 10.0 & -36.5 & 5.4 & 0.0
& 10.0  & 1.0   & 0.4 & 52.0 \\
\hline 
CAP
& 1.0 & 1092.9 & 111.5 & 0.0
& 1.0 & 6.2   & 0.4  & 30.3 \\
\hline 
CAP
& 0.1 & 1774.4 & 179.9 & 70.0
& 0.1 & 35.4   & 2.3 & 149.0 \\ 
\hline 
CAP
& 0.0 & 1588.0 & 198.1 & 80.0
& 0.0 & 26.9   & 9.3  & 184.0 \\
\hline 
CEM
& N/A & 2330.7 & 344.0 & 78.7
& N/A & 40.3   & 202.1 & 194.3 \\ 
\hline 
\end{tabular}
\caption{\textbf{CAP ablations results on HalfCheetah and Car-Racing.}} 
\label{table:cap-ablation}
\end{table*}

\subsection{Environments}
\begin{itemize}
\item \textbf{Velocity Constrained HalfCheetah:} The state space is 17-dimensional and the action space is 6-dimensional. We use the original environment reward, $v - \frac{1}{10} a^Ta$, $v$ is the forward velocity. The cost is $|v|$ \cite{zhang2020cautious}, meaning that there is a direct trade-off between cost and reward. The cost limit is set to $152$, half of the average speed of an unconstrained PPO expert agent \cite{zhang2020cautious}.
\item \textbf{Constrained Car-Racing:} The state space is a top down image of the car and the surrounding track. We downscale the image to 64 by 64 by 3. For model-free baselines, we also stack the last 4 frames, as common in reinforcement learning on image based environments. The action space is three dimensional, controlling steering, acceleration and braking. Each value is continuous and bounded. We use an action repeat of 2 to produce a better signal to the model \cite{hafner2019learning}. We keep the original reward, which incentivizes the agent to drive through as many tiles as possible. We use a binary cost that is 1 if the car skids. Skidding is a part of the original environment; a wheel skids if it's force exceeds the friction limit, which is different on grass and road surfaces.
\end{itemize}

\subsection{Uncertainty Estimators}
\textbf{State-based environments:} We model the environment transition function using an neural ensemble of size $N$, where network's output neurons parameterize a Gaussian distribution $\hat{T} = \mc{N}(\mu(s_t, a_t), \Sigma(s_t, a_t)$ \cite{chua2018deep}. We set $u(s,a) = \max_{i=1}^N \norm{\Sigma^i_{\theta}(s,a)}_{\text{F}}$ to be the maximum Frobenius norm of the ensemble standard deviation outputs, as done for offline RL in \citet{yu2020mopo}.

\textbf{Image-based environments:} We implement PlaNet \cite{hafner2019learning}, which models the environment transition function using a latent dynamics model with deterministic and stochastic transition states; we refer interested readers to the original paper for details. PlaNet does not provide an uncertainty estimate because it only utilizes a single transition model. To obtain an uncertainty estimate, we train a bootstrap ensemble of one-step hidden-state dynamics model as in \citet{sekar2020planning}. Each one-step model in the ensemble predicts, from each deterministic state $h$, the next stochastic state. We formulate our uncertainty estimator as $u(h, a) = Var({\mu_i(h, a)} | i=[1..K])$, the variance of ensemble predictions $\{\mu_i\}_{i=1}^K$. As in \citet{sekar2020planning}, to keep the scale of this uncertainty estimator similar to that of state-based uncertainty estimator, we multiply it by 10000.

\subsection{Network Architecture}
We use a neural network $C$ to approximate the environment's true cost function. When the cost is continuous, the network's output neurons parameterize a Gaussian distribution and we construct our conservative cost function as $C(s,a) + \kappa u(s,a)$. When the cost is binary, the network outputs a logit and we construct our conservative cost function as $\mathbbm{1}[C(s,a) + u(s,a) > 0]$.

To apply model free algorithms on imaged-based environments, we used a shared CNN module to encode the image input. The network consists of 5 convolutional layers followed by a ReLU non-linearity.
\begin{verbatim}
4x4 conv, 8, stride 2
3x3 conv, 16, stride 2
3x3 conv, 32, stride 2
3x3 conv, 64, stride 2
3x3 conv, 128, stride 1
\end{verbatim}

\subsection{Hyperparameters}
In Table \ref{table:cap-hyperparamters}, we include the hyperparameters we used for state-based and image-based experiments, respectively. 
\begin{table}[h]
\centering
\begin{tabular}{lrr}
\hline 
Hyperparameter & State-based & Image-based \\
\hline
Ensemble size $K$ & 5 & 5 \\
Optimizer & Adam & Adam \\
Optimizer $\kappa$ & Adam & Adam \\
Learning rate & 0.001 & 0.001 \\
Learning rate $\kappa$ & 0.1 & 0.01 \\
Initial $\kappa$ & 1.0 & 0.1 \\
Reward discount factor $\gamma$ & 0.99 & 0.99 \\
Cost discount factor $\gamma_{cost}$ & *0.99 & *0.99 \\
Batch size & 256 & 50 \\
Exploration steps & 1000 & 5000 \\
Experience buffer size & 1000000 & 1000000 \\
Uncertainty multiplier & 1 & 100000 \\
\hline
CEM Hyperparameters \\
\hline
Planning horizon $H$ & 30 & 12 \\
Max iteration $I$ & 5 & 10 \\
Population size $N$ & 500 & 1000 \\
Elite population size $E$ & 50 & 100 \\
\hline 
\end{tabular}
\caption{\textbf{CAP hyperparameters} 
} 
* We set cost discount factor to 1.0 when the cost is binary, so total cost per episode is directly interpretable.

\label{table:cap-hyperparamters}
\end{table}

\subsection{Additional results}
\label{appendix:high-dimensional-additional}
In Figure \ref{figure:cap-high-dimensional}, we illustrate the training curves of CAP and model free baselines in HalfCheetah and Car-Racing. For clarity, we focus on the first 100K/200K steps. The results are also presented in Table \ref{table:model-free-comparison}. In HalfCheetah, all model free methods have 0 cost violations in the first 100K steps, this is because they have not learnt a running gait that can violate the speed costraint. On the other hand, CAP is able to quickly a gait and keep cost below the limit, with less than two violations per 100 episodes. In CarRacing, all methods have high cost violations because the cost limit is 0. An initial random policy will violate the cost constraint and exploration will always risk violation. Even still, we see that CAP dominates FOCOPS, obtaining better episode return with lower cost and total violations. CAP has more cost violations than PPO-Lagrangian, but we see that this is because PPO-Lagrangian degrades to a trivial policy that maintains a stationary position, obtaining negative return with minimal risk of cost violations.

\subsection{Compute resources}
We use a single GTX 2080 Ti with 32 cores to run our experiments, each run takes about 10 hours in clock time.
 
\end{document}